\documentclass[conference]{IEEEtran}
\usepackage{mathptmx}
\usepackage[latin9]{inputenc}
\usepackage{float}
\usepackage{amsmath}
\usepackage{amsthm}
\usepackage{amssymb}
\usepackage{graphicx}

\makeatletter

\DeclareTextSymbolDefault{\textquotedbl}{T1}
\providecommand{\tabularnewline}{\\}
\floatstyle{ruled}
\newfloat{algorithm}{tbp}{loa}
\providecommand{\algorithmname}{Algorithm}
\floatname{algorithm}{\protect\algorithmname}

\theoremstyle{plain}
\newtheorem{thm}{\protect\theoremname}
\theoremstyle{definition}
\newtheorem{defn}[thm]{\protect\definitionname}
\theoremstyle{plain}
\newtheorem{prop}[thm]{\protect\propositionname}

\IEEEoverridecommandlockouts
\usepackage{cite}
\usepackage{amsfonts}

\usepackage{algorithm}
\usepackage{algpseudocode}
\algblock{Input}{EndInput}
\algnotext{EndInput}
\algblock{Output}{EndOutput}
\algnotext{EndOutput}

\usepackage{textcomp}
\usepackage{xcolor}

\newtheorem{assu}{Assumption}

\def\BibTeX{{\rm B\kern-.05em{\sc i\kern-.025em b}\kern-.08em
    T\kern-.1667em\lower.7ex\hbox{E}\kern-.125emX}}

\providecommand{\definitionname}{Definition}
\providecommand{\propositionname}{Proposition}
\providecommand{\theoremname}{Theorem}

\providecommand{\definitionname}{Definition}
\providecommand{\propositionname}{Proposition}
\providecommand{\theoremname}{Theorem}

\providecommand{\definitionname}{Definition}
\providecommand{\propositionname}{Proposition}
\providecommand{\theoremname}{Theorem}

\makeatother

\providecommand{\definitionname}{Definition}
\providecommand{\propositionname}{Proposition}
\providecommand{\theoremname}{Theorem}

\begin{document}
\title{Measure Contribution of Participants in Federated Learning }
\author{\IEEEauthorblockN{Guan Wang} \IEEEauthorblockA{\textit{Digital and Smart Analytics} \\
 Swiss Re\\
 Hong Kong, China \\
 guan\_wang@swissre.com} \and \IEEEauthorblockN{Charlie Xiaoqian Dang} \IEEEauthorblockA{\textit{Digital and Smart Analytics} \\
 Swiss Re\\
 Beijing, China \\
 charlie\_dang@swissre.com} \and \IEEEauthorblockN{Ziye Zhou} \IEEEauthorblockA{\textit{Research Center} \\
 Ping An P\&C Insurance Co., Ltd.\\
 Shenzhen, China \\
 zhouziye659@pingan.com.cn} }
\maketitle
\begin{abstract}
Federated Machine Learning (FML) creates an ecosystem for multiple
parties to collaborate on building models while protecting data privacy
for the participants. A measure of the contribution for each party
in FML enables fair credits allocation. In this paper we develop simple
but powerful techniques to fairly calculate the contributions of multiple
parties in FML, in the context of both horizontal FML and vertical
FML. For Horizontal FML we use deletion method to calculate the grouped
instance influence. For Vertical FML we use Shapley Values to calculate
the grouped feature importance. Our methods open the door for research
in model contribution and credit allocation in the context of federated
machine learning. 
\end{abstract}


\section{Introduction}

Federated Learning or Federated Machine Learning (FML) \cite{b1}
is introduced to solve privacy issues in machine learning using data
from multiple parties. Instead of transferring data directly into
a centralized data warehouse for building machine learning models,
Federated Learning allows each party to own the data in its own place
and still enables all parties to build a machine learning model together.
This is achieved either by building a meta-model from the sub-models
each party builds so that only model parameters are transferred, or
by using encryption techniques to allow safe communications in between
different parties \cite{b2}.

Federated Learning opens new opportunities for many industry applications.
Companies have been having big concerns on the protection of their
own data and are unwilling to share with other entities. With Federated
Learning, companies can build models together without disclosing their
data and share the benefit of machine learning. An example of Federated
Learning use case is in insurance industry. Primary insurers, reinsurers
and third-party companies like online retailers can all work together
to build machine learning models for insurance applications. Number
of training instances is increased by different insurers and reinsurers,
and feature space for insurance users is extended by third-party companies.
With the help of Federated Learning, machine learning can cover more
business cases and perform better.

For the the ecosystem of Federated Learning to work, we need to encourage
different parties to contribute their data and participate in the
collaboration federation. A credit allocation and rewarding mechanism
is crucial for the incentives current and potential participants of
Federated Learning. A fair measure of the contribution for each party
in Federated Learning enables fair credits allocation. Data quantity
alone is certainly not enough, as one party may contribute lots of
data that doesn't help much on building the model. We need a way to
fairly measure the data quality overall and hence decide the contribution.

In this paper we develop simple but powerful techniques to fairly
calculate the contributions of multiple parties in FML, in the context
of both horizontal FML and vertical FML. For Horizontal FML, each
party contributes part of the training instances. We use deletion
method to calculate the grouped instance influence. Each time we delete
the instances provided from one certain party and retrain the model,
and calculate the difference of the prediction results between the
new model and the original one, and use this measure of difference
to decide the contribution of this certain party. For Vertical FML,
each party owns party of the feature space. We use Shapley Values
\cite{b21} to calculate the grouped feature importance, and use this
measure of importance to decide the contribution of each party. The
method we propose in our knowledge is the first attempt of research
on model contribution and credit allocation in the context of federated
machine learning.

In the next chapters of this paper, we first briefly introduce Federated
Learning. We then cover the Federated Deletion method and Federated
Shap method we propose on measuring contributions of multiple parties
for horizontal and vertical FML models, followed by some experiments.
We conclude the paper with some discussions in the last chapter.

\section{Federated Learning}

Federated Learning originated from some academic papers like \cite{b1,b3}
and a follow-up blog from Google in 2017. The idea is that Google
wants to train its own input method on its Android phones called ``Gboard''
but does not want to upload the sensitive keyboard data from their
users to Google's own servers. Rather than uploading user's data and
training models in the cloud, Google lets users train a separate model
on their own smartphones (thanks to the neural engines from several
chip manufacturers) and upload those black-box model parameters from
each of their users to the cloud and merge the models, update the
official centralized model, and push the model back to Google users.
This not only avoids the transmission and storage of user's sensitive
personal data, but also utilizes the computational power on the smartphones
(a.k.a the concept of Edge Computing) and reduce the computation pressure
from their centralized servers.

When the concept of Federated Learning was published, Google's focus
was on the transmission of models as the upload bandwidth of mobile
phones is usually very limited. One possible reason is that similar
engineering ideas have been discussed intensively in distributed machine
learning. The focus of Federated Learning was thus more on the ``engineering
work'' with no rigorous distributed computing environment, limited
upload bandwidth and slave nodes as massive number of users.

Data privacy is becoming an important issue, and lots of relating
regulations and laws have been taken into action by authorities and
governments\cite{b4,b5}. The companies that have been accumulating
tons of data and have just started to make value of it now have their
hands tightened. On the other hand, all companies value a lot their
own data and feel reluctant from sharing data with others. Information
islands kill the possibility of cooperation and mutual benefit. People
are looking for a way to break such prisoner dilemma while complying
with all the regulations. Federated Learning was soon recognized as
a great solution for encouraging collaboration while respecting data
privacy.

\cite{b2} describes Federated Learning in three categories: Horizontal
Federated Learning, Vertical Federated Learning and Federated Transfer
Learning. Such categorization extends the concept of Federated Learning
and clarify the specific solutions under different use cases.

\emph{Horizontal Federated Learning} applies to circumstances where
we have a lot of overlap on features but only a few on instances.
This refers to the Google Gboard use case and models can be ensembled
directly from the edge models.

\emph{Vertical Federated Learning} refers to where we have many overlapped
instances but few overlapped features. An example is between insurers
and online retailers. They both have lots of overlapped users, but
each owns their own feature space and labels. Vertical Federated Learning
merges the features together to create a larger feature space for
machine learning tasks and uses homomorphic encryption to provide
protection on data privacy for involved parties.

\emph{Federated Transfer Learning} uses Transfer Learning \cite{b6}
to improve model performance when we have neither much overlap on
features nor on instances.

An example in insurance industry to illustrate the idea is the following.
Horizontal FML corresponds to primary insurers working with a reinsurer.
For the same product primary insurers share the similar features.
Vertical FML corresponds to reinsurer working with another third-party
data provider like online retailer. An online retailer will have more
features for a certain policyholder that can increase the prediction
power for models built for insurance.

For a detailed introduction of Federated Learning and their respective
technology that is used, please refer to \cite{b2}.

\section{Deletion Method for Horizontal FML}

Most model interpretation methods can be applied, with some minor
modifications, to contribution measure for Horizontal Federated Learning
as all parties have data for the full feature space. There is no special
issue for interpreting prediction results on both training data and
new data, for both specific single predictions as granular check or
for batch pre-dictions as holistic check.

Approaches to identifying influential instances, such as deletion
diagnostics \cite{b8} and influence functions \cite{b7}, can be
used to measure the importance of individuals to a machine learning
model. Here we propose a method based on deletion diagnostics to measure
contributions of different parties for horizontal FML.

Deletion diagnostics is intuitive. With the deletion approach, we
retrain the model each time an instance is omitted from training dataset
and measure how much the prediction of retrained model changes. Supposing
we are evaluating the effect of the $i$th instance on the model predictions,
the influence measure can be formulated as follows, 
\begin{equation}
\text{Influence}^{-i}=\frac{1}{n}\sum_{j=1}^{n}|\hat{y}_{j}-\hat{y}_{j}^{-i}|,
\end{equation}
where $n$ is the size of dataset, $\hat{y}_{j}$ is the prediction
on $j$th instance made by the model trained on all data, and $\hat{y}_{j}^{-i}$
is the prediction on $j$th instance made by model trained with the
$i$th instance omitted.

For one party in horizontal FML with a subset of instances $D$, we
define the contribution as the total influence of all instances it
posseses in the following form, 
\begin{equation}
\text{Influence}^{-D}=\sum_{i\in D}\text{Influence}^{-i}.\label{eq:influenceD}
\end{equation}
We propose an approximation algorithm to implement the above influence
measure, considering a batch of instances as a whole during each deletion,
which is shown in Algorithm \ref{algo:influence}.

\begin{algorithm}
\caption{Approximating influence estimation for each party in horizontal FML}
\label{algo:influence} \begin{algorithmic} \Input \State number
of parties $K$, model $f$ \State instance subsets $D_{1},\dots,D_{K}$
\EndInput \Output \State Influence measure $\text{Influence}^{-D_{k}}$
for $k=1,\dots,K$ \EndOutput \ForAll{k=1,\dots ,K} \State delete
$D_{k}$ from training dataset \State retrain model $f'$ \State
compute $\text{Influence}^{-D_{k}}=\frac{1}{n}\sum_{j}|\hat{y}_{j}-\hat{y}_{j}^{-D}|$
\EndFor \State return $\text{Influence}^{-D_{k}}$ for $k=1,\dots,K$
\end{algorithmic} 
\end{algorithm}

\section{Shapley Values for Vertical FML}

In this section we focus on the contribution measure of different
parties in vertical Federated Machine Learning. In the vertical mode
a party contributes to FML model by sharing its features with other
parties, which means the contribution of the party can be represented
by the combined contributions of its shared features. Therefore, we
first introduce how to distribute the contributions among individual
features and then show the extension to measuring contribution of
grouped features.

\subsection{Shapley Values for Individual Feature}

Generally, we are interested in how a particular feature value influences
the model prediction. For an additive model like linear regression
\begin{equation}
f(x)=\beta_{0}+\sum_{i=1}^{n}\beta_{i}x_{i},
\end{equation}
where $\beta_{i}$ is the model coefficient and $x_{i}$ the feature
value, we can measure the influence of $X_{i}=x_{i}$ according to
the situational importance \cite{b10} 
\begin{equation}
\varphi_{i}(x)=\beta_{i}x_{i}-\beta_{i}\mathbb{E}[X_{i}].
\end{equation}
The situational importance is the difference between what a feature
contributes when its value is $x_{i}$ and what it is expected to
contribute. For a more general model which we treat as a black box,
the feature influence can be computed in the following way similar
to the situational importance: 
\begin{equation}
\varphi_{i}(x)=f(x_{1},\dots,x_{n})-\mathbb{E}[f(x_{1},\dots,x_{i},\dots,x_{n})],\label{eq:generalSituationImportance}
\end{equation}
which is the difference between a prediction for an instance and the
expected prediction for the same instance if the \textit{i}th feature
had not been known.

The Shapley value \cite{b11}, which is originated from coalitional
game theory with proven theoretical properties, provides an effective
approach to distribute contributions among features in a fair way
by assigning to each feature a number which denotes its influence
\cite{b12}\cite{b13}\cite{b14}\cite{b15}. In a coalitional game,
it is assumed that a grand coalition formed by \textit{n} players
has a certain worth and each smaller coalition has its own worth.
The goal is to ensure that each player receives his fair share, taking
into account all sub-coalitions. In our case, the Shapley value is
defined as 
\begin{equation}
\phi_{i}(x)=\sum\limits _{Q\subseteq S\setminus\{i\}}\frac{|Q|!(|S|-|Q|-1)!}{|S|!}(\Delta_{Q\cup\{i\}}(x)-\Delta_{Q}(x)),\label{eq:shapleyValue}
\end{equation}
where $S$ is s feature index set, $Q\subseteq S=\{1,2,\dots,n\}$
is a subset of features, $x$ is the vector of feature values of the
instance in consideration and $|\cdot|$ is the size of a feature
set. $\Delta_{Q}(x)$ denotes the influence of a subset of feature
values, which generalizes \eqref{eq:generalSituationImportance},
in the following form 
\begin{equation}
\Delta_{Q}(x)=\mathbb{E}[f|X_{i}=x_{i},\forall i\in Q]-\mathbb{E}[f].
\end{equation}

The Shapley value $\phi_{i}(x)$ gives a strong solution to the problem
of measuring individual feature contribution. However, computing \eqref{eq:shapleyValue}
has an exponential time complexity, making the method infeasible for
practical scenarios. An approximation algorithm with Monte-Carlo sampling
is proposed in \cite{b12} to reduce the computational complexity:
\begin{equation}
\phi_{i}(x)=\frac{1}{M}\sum\limits _{m=1}^{M}\left(f(x_{+i}^{m})-f(x_{-i}^{m})\right),\label{eq:shapleyValueApprox}
\end{equation}
where $M$ is the number of iterations. $f(x_{+i}^{m})$ is the prediction
for instance $x$, with a random number of feature values replaced
by feature values from a randomly selected instance $z$, except for
the respective value of feature $i$. The vector $x_{-i}^{m}$ is
almost identical to $x_{+i}^{m}$, except that the value $x_{i}^{m}$
is taken from the sampled $z$. The approximation algorithm is summaried
in Algorithm \ref{algo:individualShapley}.

\begin{algorithm}
\caption{Approximating Shapley estimation for individual feature value}
\label{algo:individualShapley} \begin{algorithmic} \Input \State
number of iterations $M$, instance feature vector $x$, \State model
$f$, feature space $\mathcal{X}$, and feature index $i$ \EndInput
\Output \State Shapley value for the value of the $i$th feature
$\phi_{i}(x)$ \EndOutput \ForAll {$m=1,\dots,M$} \State select
a random instance $z\in\mathcal{X}$ \State select a random permutation
of the feature values \State construct two new instances: \State
~~ $x_{+i}=(x_{1},\dots,x_{i-1},x_{i},z_{i+1},\dots,z_{n})$ \State
~~ $x_{-i}=(x_{1},\dots,x_{i-1},z_{i},z_{i+1},\dots,z_{n})$ \State
compute marginal contribution $\phi_{i}^{m}=f(x_{+i})-f(x_{-i})$
\EndFor \State compute Shapley value $\phi_{i}(x)=\frac{1}{M}\sum_{m=1}^{M}\phi_{i}^{m}$
\end{algorithmic} 
\end{algorithm}

\subsection{Shapley Values for Grouped Features}

Vertical Federated Learning raises new issues for measuring contributions
of multiple parties where the feature space is divided into different
parts. Directly using methods like Shapley values for each prediction
will very likely reveal the protected feature value from the other
parties and cause privacy issues. Thus it is not trivial to develop
a safe mechanism for vertical Federated Learning and find a balance
between contribution measurement and data privacy.

We propose a variant version of the approach proposed in \cite{b13}
to use Shapley value for measuring contributions of different parties
in vertical FML. Here we take the dual-party vertical Federated Learning
as an example, while the idea can be extended to multiple parties.
For the $k$th instance, the label is $y_{k}$ and one party owns
part of the features $x^{h,k}$ and the other party owns the rest
part of the features $x^{g,k}$, where $k=1,\dots,K$ as we suppose
both parties have $K$ overlapped instances with the same IDs. By
using vertical FML, the two parties collaborate to develop a machine
learning model for predicting labels $Y$. We first give some definitions
and assumptions in this problem and then propose an approximation
algorithm to compute the Shapley group value for measuring the contributions
of different parties. 
\begin{defn}
(United Federated Feature). For the vertical FML with a set of parties
$G$ and a set of features $S$, the united federated feature $x^{fed}$
is a combination feature of the features $x^{g}\in X^{g}\subset S$
for party $g\in G$. 
\end{defn}
\noindent We treat a united federated feature as a single feature
since individual features of each party are private and not visible
to other parties. 
\begin{defn}
(Shapley Group Value). The Shapley group value is the group value
that sums the individual Shapley values for all elements in the group.
Formally, the Shapley group value for a subset $P\subset S$ is given
by 
\begin{equation}
\phi_{P}(x)=\sum\limits _{i\in P}\phi_{i}(x).\label{eq:groupShapley}
\end{equation}
\end{defn}
\noindent The Shapley group value denotes the contribution of a subset
of features. 
\begin{defn}
(Shapley Group Interaction Index). The Shapley group interaction index
is the additional combined feature effect of group $P\subset G$ given
by 
\begin{equation}
\varphi_{P}(x)=\sum\limits _{Q\subseteq S\setminus P}\frac{|Q|!(|S|-|Q|-1)!}{|S|!}\delta_{P}(x),
\end{equation}
where 
\begin{equation}
\delta_{P}(x)=\Delta_{Q\cup P}(x)-\sum\limits _{i\in P}\Delta_{Q\cup\{i\}}(x)+(|P|-1)\Delta_{Q}(x).\label{eq:deltaInteraction}
\end{equation}
\end{defn}
\noindent The Shapley group interaction index is a variant of the
Shapley interaction index \cite{b16} which extends the definition
of the combined feature effect from two features to a group of features.

\begin{assu}\label{ass:interaction} The Shapley group interaction
index for feature set $X^{g}\subset S$ of any party in vertical FML
is zero, i.e., $\varphi_{X^{g}}(x)=0$, $\forall g\in G$. \end{assu}

\begin{assu}\label{ass:dummy} All features in the feature set $X^{g}$
of party $g$ are dummy features with $\Delta_{Q\cup\{j\}}(x)=\Delta_{Q}(x)+\Delta_{\{j\}}(x)$,
$\forall j\in X^{g}$, $\forall g\in G$ and $\forall Q\subset S$.
\end{assu} 
\begin{prop}
\label{prop:groupShap} If either of Assumption \ref{ass:interaction}
and \ref{ass:dummy} holds, then the Shapley group value for a party
$g\in G$ with feature set $X^{g}$ is given by 
\begin{equation}
\phi_{X^{g}}=\sum\limits _{Q\subseteq S\setminus\{j^{fed}\}}\frac{|Q|!(|S|-|Q|-1)!}{|S|!}(\Delta_{Q\cup\{j^{fed}\}}(x)-\Delta_{Q}(x)),\label{eq:guestGroupShapley}
\end{equation}
where $j^{fed}$ is the index of the united federated feature $x^{fed}$. 
\end{prop}
\begin{proof} We consider the vertical FML scenario where the other
parties act collaboratively as a whole and reach an agreement on a
protocol of sharing and permulation among all their features when
computing the Shapley group value for one party. Actually this reduces
to the dual-party FML case. If Assumption \ref{ass:interaction} holds,
then 
\[
\sum\limits _{i\in P}(\Delta_{Q\cup\{i\}}(x)-\Delta_{Q}(x))=\Delta_{Q\cup P}(x)-\Delta_{Q}(x).
\]
According to the definition of federated feature, we can treat $\{j^{fed}\}$
as $X^{g}$. Thus putting the above equation into \eqref{eq:shapleyValue}
and \eqref{eq:groupShapley} gives \eqref{eq:guestGroupShapley}.
If Assumption \ref{ass:dummy} holds, then \begin{IEEEeqnarray}{rCl} \Delta_{Q \cup \{j_1^g,\dots ,j_k^g\}} (x) & = & \Delta_{Q \cup \{j_1^g,\dots ,j_{k-1}^g\}} (x) + \Delta_{\{j_k^g\}} (x) \nonumber \\  & = & \dots \nonumber \\  & = & \Delta_{Q} (x) + \sum_{j \in X^g} \Delta_{\{j\}} (x) \nonumber \end{IEEEeqnarray} 

The above equation together with the dummy property makes \eqref{eq:deltaInteraction}
equal to zero. Thus Assumption \ref{ass:interaction} holds. \end{proof}

Proposition \ref{prop:groupShap} indicates that we can measure the
importance of a feature subset without revealing the details of any
private feature of a party in the vertical FML. Suppose we want to
measure the contribution of one party to the prediction of an instance
by looking at the Shapley group value of the feature set shared by
the party. Instead of giving out individual Shapley values for all
features in its feature space, we combine the private features as
one united federated feature, and compute the Shapley value for this
federated feature together with the features of all the other parties.

Since this method requires to turn on and off certain features for
calculating the Shapley value, for the federated feature we will need
a specific ID to inform the party in consideration to return its part
of the prediction with all its features turned off. For models that
takes in NA values, this mean that the features will be set to NA.
For models that cannot handle missing values, we follow the practice
of \cite{b13} to set the feature values to be the median value of
all the instances as the reference value.

Although the two assumptions are quite strong, the experiment results
show that the approximation algorithm works well in real scenarios.
Also, as discussed above the approximation algorithm for the dual-party
case can be extended to measuring contributions of multiple parties
as long as an agreement on a protocol of feature sharing and permutation
is reached. In summary, the process of computing Shapley group value
for one party is described in Algorithm \ref{algo:groupShapley}.
Repeating the estimation algorithm for all parties gives their corresponding
contribution measure.

\begin{algorithm}
\caption{Approximating Shapley estimation for federated feature for one party
in vertical FML}
\label{algo:groupShapley} \begin{algorithmic} \Input \State number
of iterations $M$, instance feature vector $x$, \State model $f$,
the federated feature index $j^{fed}$ , \State the index set of
other features $I^{h}$ , party $g$ \EndInput \Output \State Shapley
group value $\phi_{j^{fed}}$ for $j^{fed}$ \EndOutput \ForAll
{$m=1,\dots,M$} \State select a subset $Q\in I^{h}\cup\{j^{fed}\}$
\State construct new instance $x'$: \State ~~ Set $x'_{k}=x_{k}$
for $k\in Q$ \State ~~ Set $x'_{k}$ to reference value for $k\notin Q$
\If{$j^{fed}\in Q$} \State Send encrypted ID of $x$ to party
$g$ \State Set $x'_{j^{fed}}=x_{j^{fed}}$ \Else \State Send special
ID to party $g$ \State Set $x'_{j^{fed}}$ to reference value \EndIf
\State Run federated model prediction for $x'$ \State Save prediction
result of $Q$ \EndFor \State compute $\phi_{j^{fed}}$ using Algorithm
\ref{algo:individualShapley} \State return Shapley group value $\phi_{j^{fed}}$
\end{algorithmic} 
\end{algorithm}

\section{Experiment}

We developed the algorithms for calculating multi-party contribution
as derived in Section 3 and 4. In this section, we test our algorithm
by training a machine learning model on the Cervical cancer (Risk
Factors) Data Set \cite{b9} and experimenting by calculating the
participant contributions on both horizontal and vertical FML setups.
The Cervical cancer dataset is used to predict whether an individual
female will get cervical cancer given indicators and risk factors
as features of the dataset, which is showed in Table~\ref{tab1}.
We normalized the data and used Scikit-learn to train a SVM (Support
Vector Machine) model for the cervical cancer classification task.

\begin{table}[htbp]
\caption{Cervical cancer (Risk Factors) Data Set Attribute Information}

\centering{}%
\begin{tabular}{|c|c|}
\hline 
\textbf{Attribute}  & \multicolumn{1}{c|}{\textbf{Type}}\tabularnewline
\hline 
Age  & int\tabularnewline
\hline 
Number of sexual partners  & int\tabularnewline
\hline 
First sexual intercourse (age)  & int\tabularnewline
\hline 
Num of pregnancies  & int\tabularnewline
\hline 
Smokes  & bool\tabularnewline
\hline 
Smokes (years)  & int\tabularnewline
\hline 
Hormonal Contraceptives  & bool\tabularnewline
\hline 
Hormonal Contraceptives (years)  & int\tabularnewline
\hline 
IUD  & bool\tabularnewline
\hline 
IUD (years)  & int\tabularnewline
\hline 
STDs  & bool\tabularnewline
\hline 
STDs (number)  & int\tabularnewline
\hline 
STDs: Number of diagnosis  & int\tabularnewline
\hline 
STDs: Time since first diagnosis  & int\tabularnewline
\hline 
STDs: Time since last diagnosis  & int\tabularnewline
\hline 
Biopsy$^{\mathrm{a}}$  & bool\tabularnewline
\hline 
\multicolumn{2}{l}{$^{\mathrm{a}}$Target Variable}\tabularnewline
\end{tabular}\label{tab1} 
\end{table}

The contribution of the FML participants can be of two ways, namely
the horizontal FML setup where we use deletion method for indicating
group instance importance, and the vertical FML setup where Shapley
value is used for evaluating importance of features from different
parties. For purpose of illustration, we separately considered the
experiments in order to illustrate that for both horizontal FML and
vertical FML our proposed methods can give a reasonable explanation
for the contribution of multiple participants.

\subsection{Deletion Method (Horizontal FML)}

As we explained in Section 3, the deletion method can be used to evaluate
importance of a single instance and then generalized to the scenario
where different groups of instances are from different parties. The
experiments are performed on the Cervical dataset with ``Biopsy''
as the target value. For simplicity, we only considered binary classification
problem, where the Biopsy takes the value Health and Cancer. Since
the deletion method is defined for the training process, without losing
the generality we don't split the Cervical dataset in the experiment,
so the training set has the number of instances the same as the entire
dataset. We use SVM as the classification algorithm, where the output
is set to 'probability' and the kernel is RBF (Radial Basis Function)
with the coefficient as 1/(number of features). In order to simulate
the participant contribution, we evenly split the dataset with a given
number of instances, so that in our experiment we conceptually build
up a Horizontal Federated Learning ecosystem with five players. We
used deletion method to calculate the contribution of those five participants
and the results are shown in Fig.~\ref{fig}.

\begin{figure}[htbp]
\centerline{\includegraphics[scale=0.4]{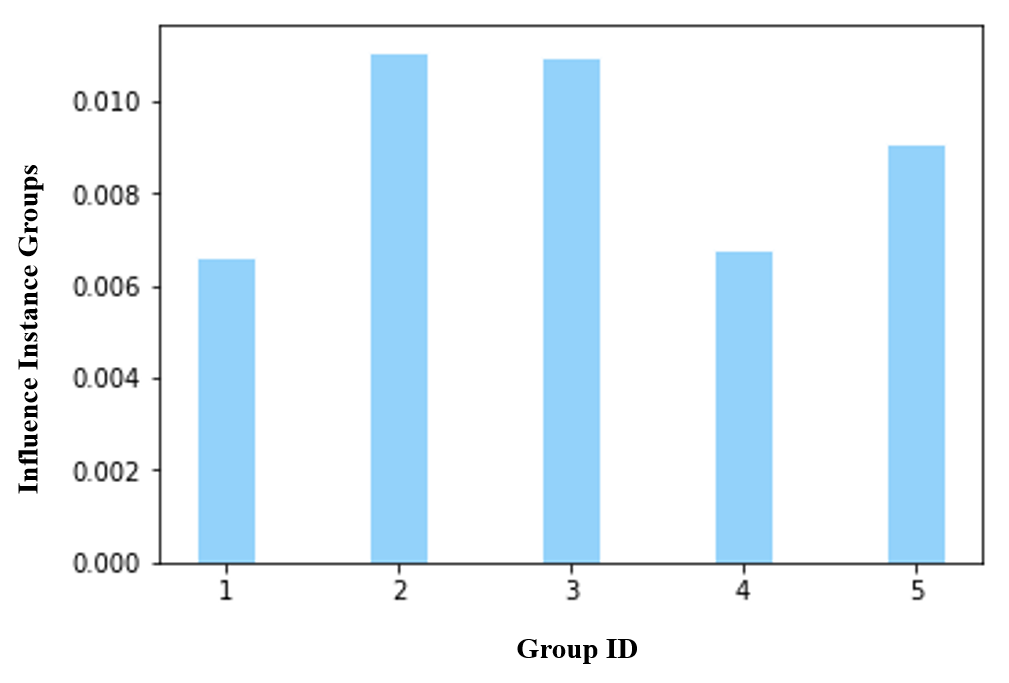}} \caption{The importance of instance groups of cervical data. We simulated five
parties and each party has same number of training instances. The
vertical axis shows the value of horizontal FML instance group importance
value.}
\label{fig} 
\end{figure}

\subsection{Shapley Value (Vertical FML)}

We also did the experiments to calculate the Shapley value for feature
importance, where we simulate the vertical FML ecosystem and each
participant shares a certain part of the feature space. The experiments
are performed on the same Cervical cancer dataset as explained in
the previous session. We randomly shuffled the data and used 70\%
instances for training and 30\% for testing. For testing data the
accuracy reaches 95.42\%. The algorithm's setup is exactly the same
as the experiment in the previous session. In order to avoid the inconsistence
due to the algorithm's hyperparameter choices, the random state for
splitting and shuffling the dataset is set to the same random seed.

As the first demonstration, we pick one specific instance from the
training data, run the prediction and test our Shapley Federated algorithm
for feature importance. The result can be seen on Fig.~\ref{fig2}.

\begin{figure}[htbp]
\centerline{\includegraphics[scale=0.5]{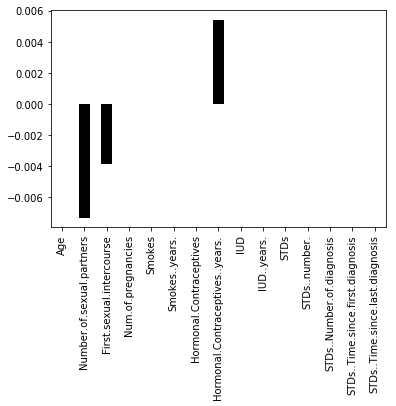}} \caption{The Shapley value for predicting one instance in Cervical cancer dataset.
This is the demonstration that Shapley value can give a reasonable
explanation for feature importance on each prediction.}
\label{fig2} 
\end{figure}

Following the demonstration, we then considered the Vertical FML ecosystem
framework. We first calculated the Shapley value for the whole feature
space, as shown in Fig.~\ref{fig3}, which directly reflects the
importance for different features as normal Shapley value indicats.

\begin{figure}[htbp]
\centerline{\includegraphics[scale=0.6]{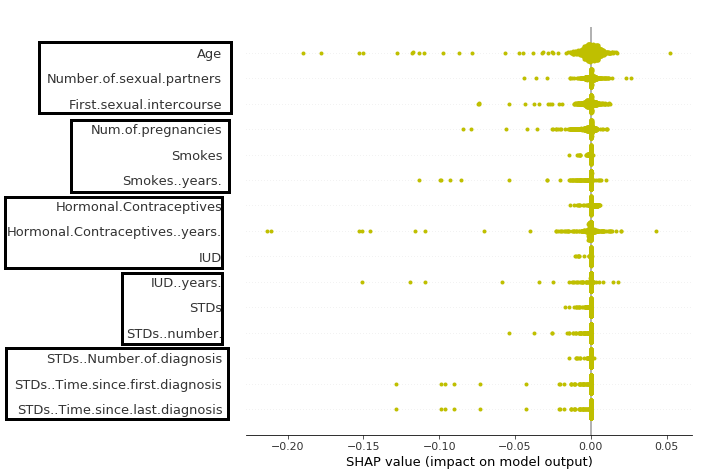}}\centerline{ \includegraphics[scale=0.6]{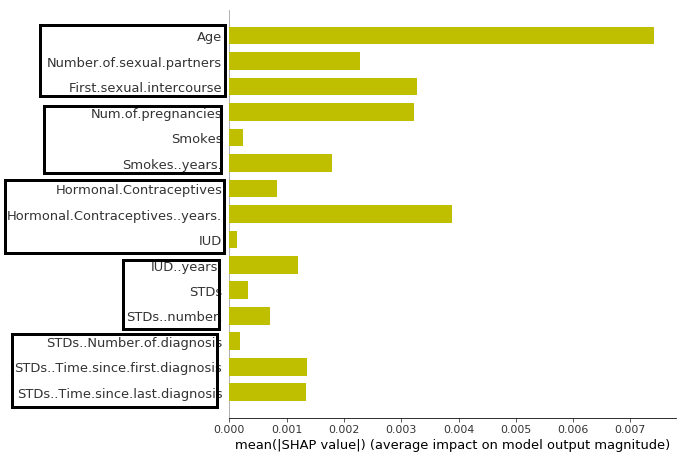}}\caption{Feature Importance (Shapley values) for 855 instances for the whole
feature space. Above one is the scatter plot for each prediction on
the data, below one is bar plot for each feature's total contribution
from all the predictions.}
\label{fig3}
\end{figure}

We then simulated the vertical FML for multiple participants, where
we evenly separate the 15 features into 5 groups and each group represents
a single participants with 3 features. Each time we group the features
from one party together as the federated feature, and run the Shapley
value algorithms to calculate the feature importance for this single
federated feature together with the other individual features from
the other participants. The simulated results are showed in Fig.~\ref{fig5}
and Fig.~\ref{fig6}. Another way of doing this is to use federated
features all together for all the participants and calculate the Shapley
value at one go. We expect that this will give less accurate results.
Our experiment indicates that in the multi-party Vertical FML setup,
federated Shapley value is a good quantity to indicate the contribution
for each participant.

\begin{figure}[htbp]
\centerline{\includegraphics[scale=0.4]{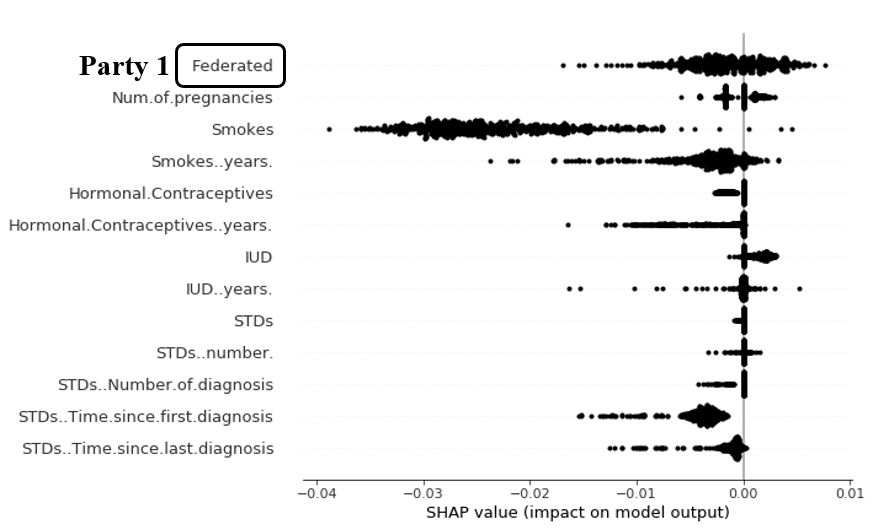}} \centerline{\includegraphics[scale=0.4]{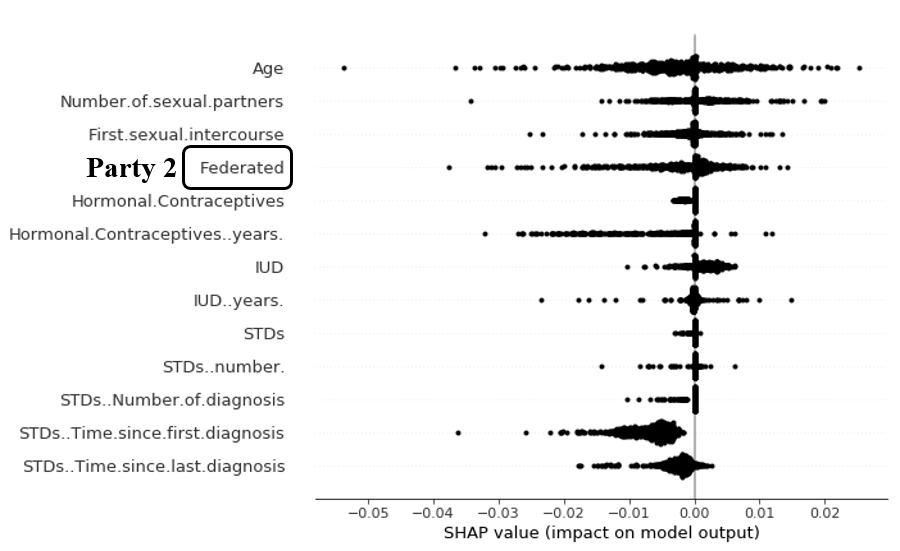}}
\centerline{\includegraphics[scale=0.4]{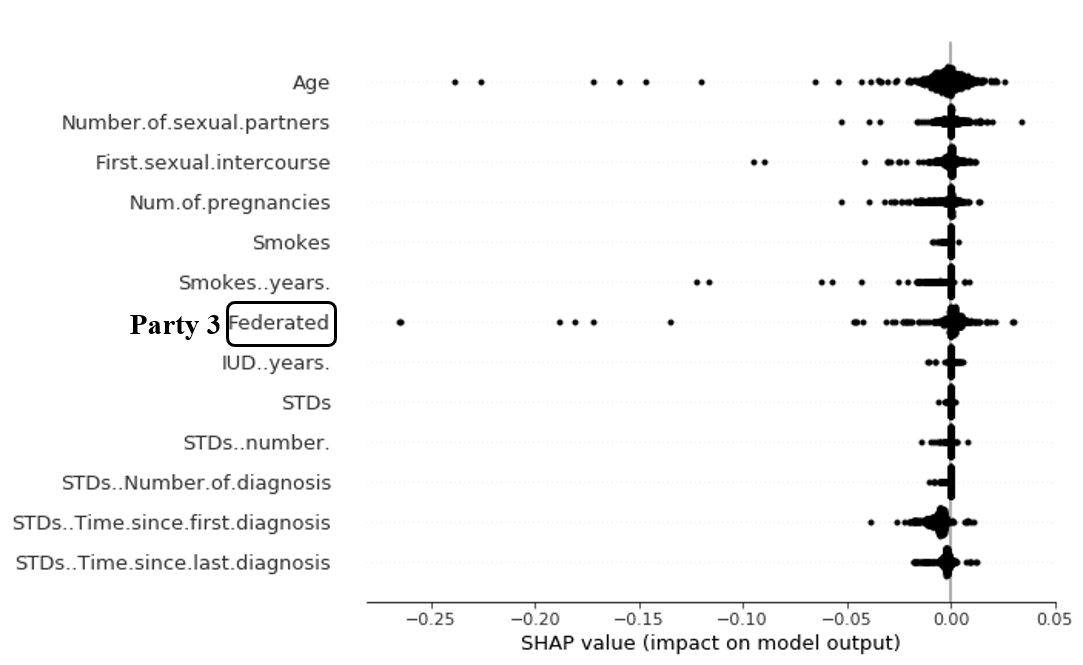}} \centerline{\includegraphics[scale=0.4]{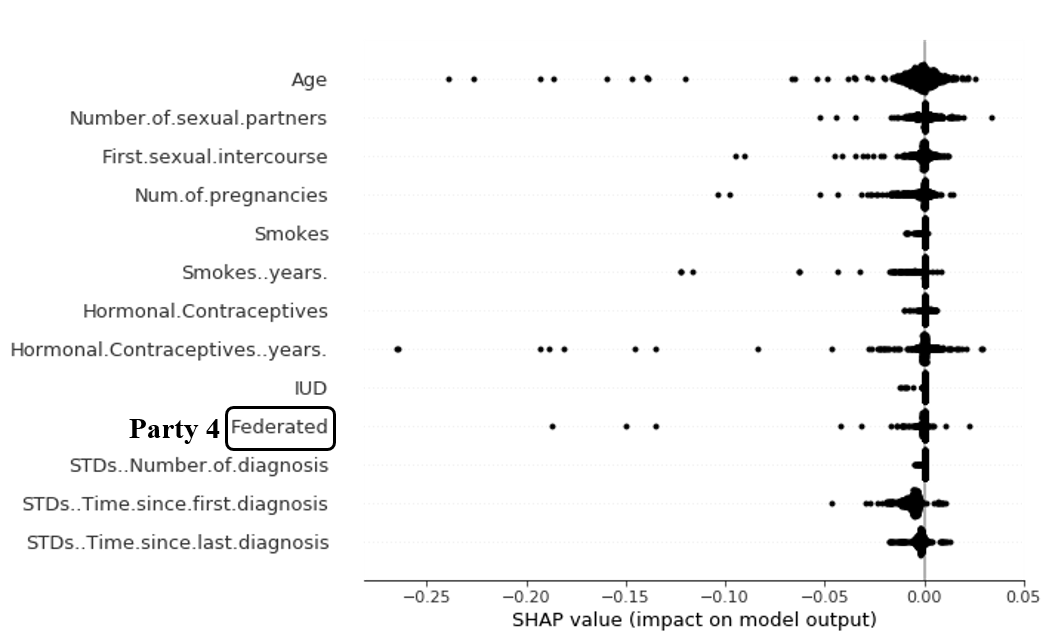}}
\centerline{\includegraphics[scale=0.4]{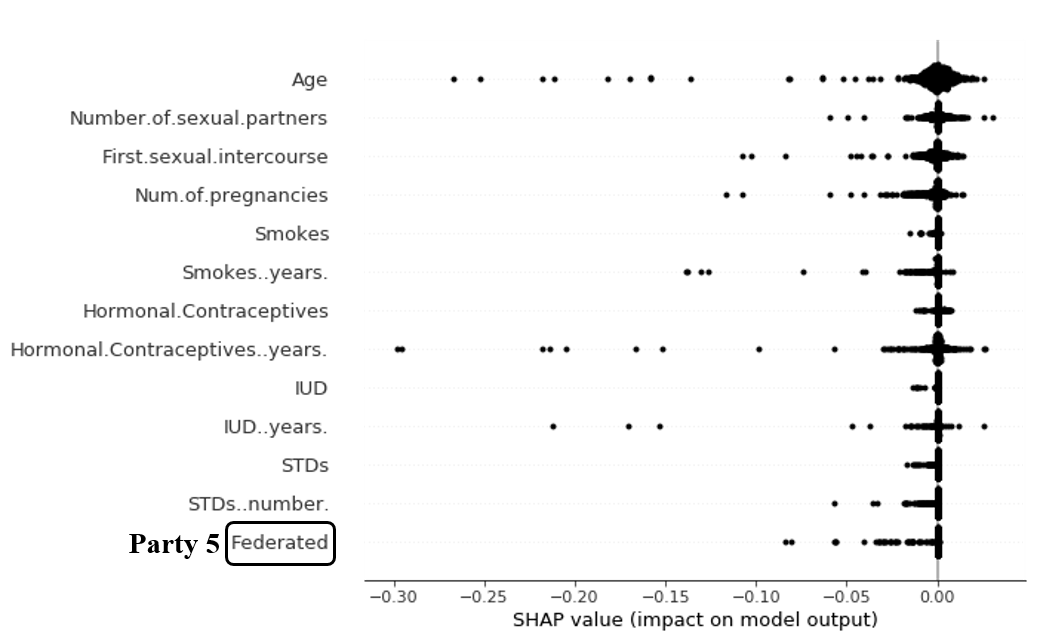}} \caption{Scatter plot for Feature Importance (Shapley values) for 855 instances.
We considered different federated groups of different features. For
combined feature has different impact on the feature importance. We
evenly separated the 15 features into 5 groups, and each group has
3 features }
\label{fig5} 
\end{figure}

\begin{figure}[htbp]
\centerline{\includegraphics[scale=0.4]{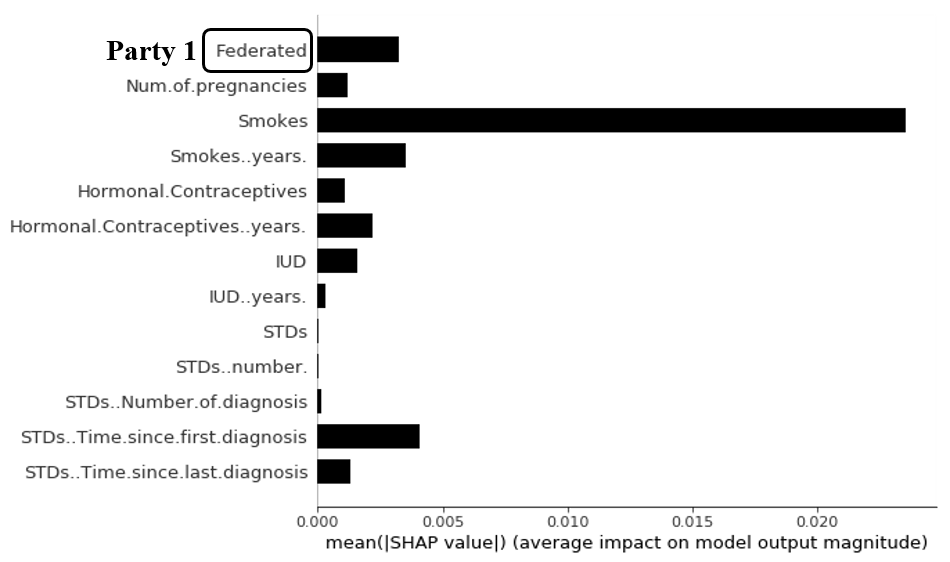}} \centerline{\includegraphics[scale=0.4]{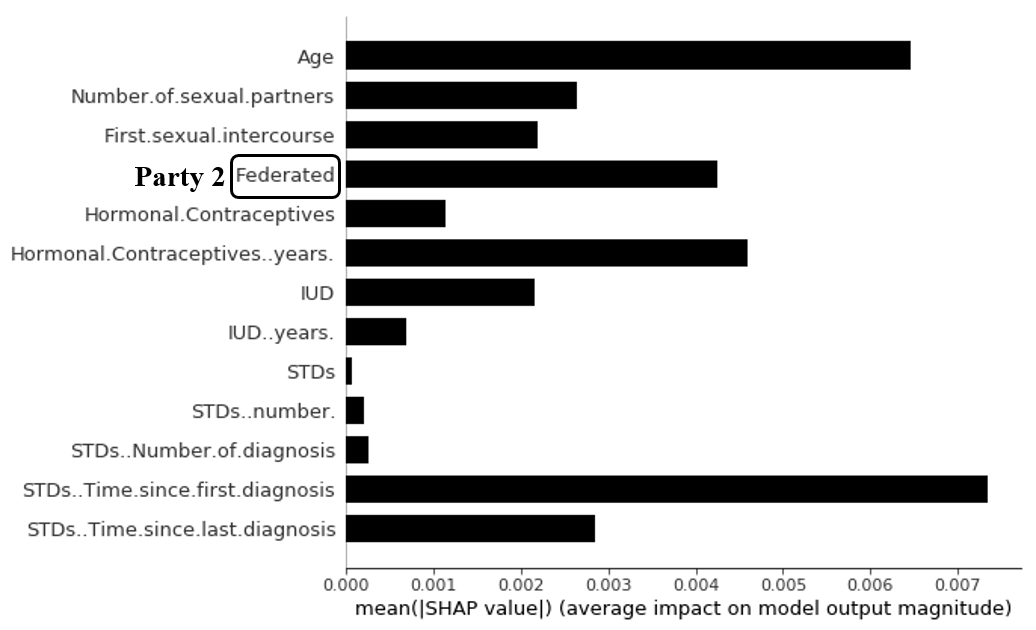}}
\centerline{\includegraphics[scale=0.4]{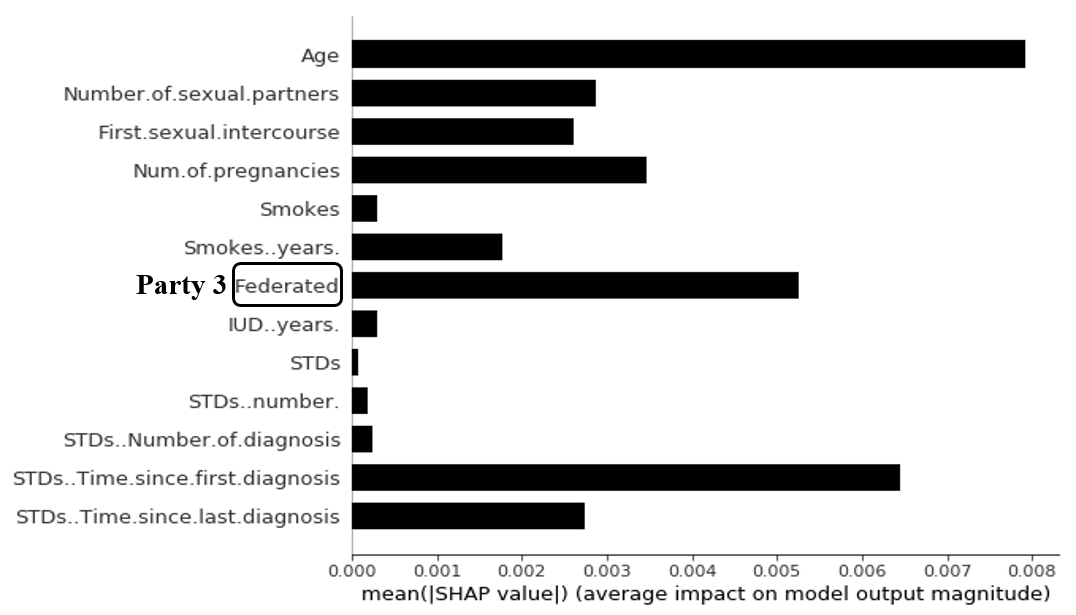}} \centerline{\includegraphics[scale=0.4]{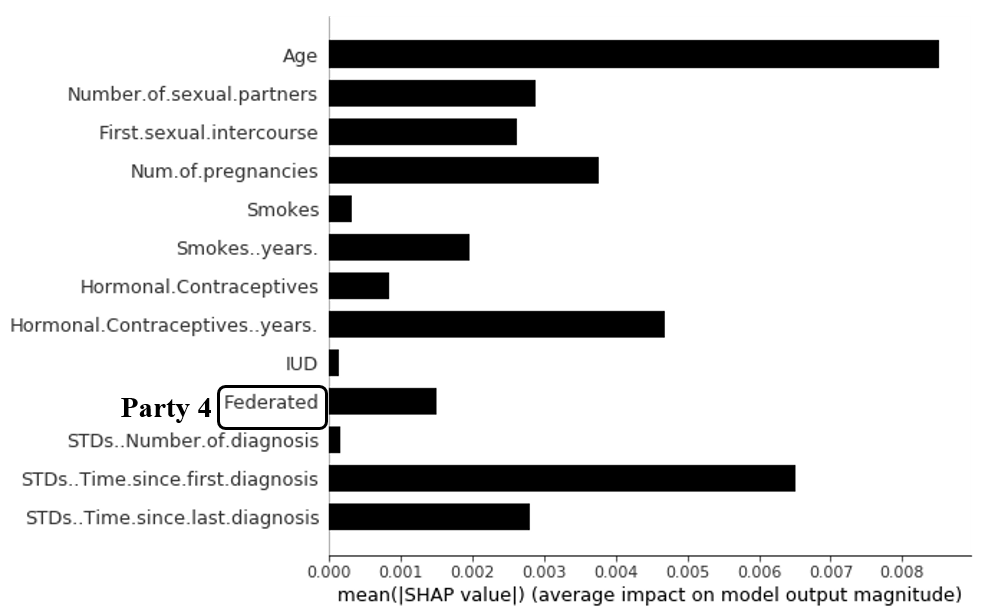}}
\centerline{\includegraphics[scale=0.4]{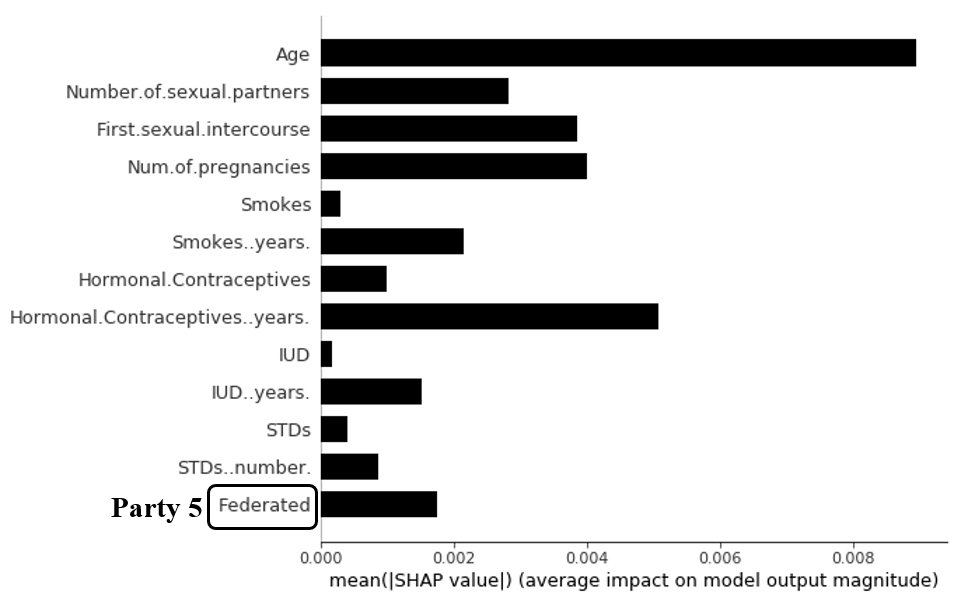}} \caption{Bar plot for average Feature Importance (Shapley values) for 855 instances.
We considered different federated groups of different features. For
combined feature has different impact on the feature importance. We
evenly separated the 15 features into 5 groups, and each group has
3 features}
\label{fig6} 
\end{figure}

\section{Conclusion}

Fair calculating the contribution of each participants in Federated
Machine Learning is crucial for credits and rewards allocation. In
this paper, we proposed methods that can calculate participant contribution
for both Horizontal FML and Vertical FML by using group instance deletion
and group Shapley values. Our experiment results indicate that our
method is effective and can give fair and reliable contribution measurements
for FML participants without disclosing the data and breaking the
initial intent of preserving data privacy.

Our work for contribution measurement for FML participants is model
agnostic, meaning that this should work for almost any kind of machine
learning algorithms, and become a general framework for this task.
We expect our work can be built into FML tool sets like FATE and TFF
and become the start of developing a standard model contribution measurement
module for Federated Learning that is critical for industrial applications.

For future work, we expect some more advanced algorithms like influential
functions \cite{b7} for horizontal FML and for some sampling version
of calculating Shapley values for vertical FML. Those algorithms will
help get an accurate and fair contribution measurement results with
higher computational efficiency.

\bigskip{}

\bigskip{}

\bigskip{}

\bigskip{}

\bigskip{}

\bigskip{}

\bigskip{}

\bigskip{}

\bigskip{}

\bigskip{}

\bigskip{}


\begin{thebibliography}{10}
\bibitem{b1} Jakub Konen, H. Brendan McMahan, Felix X. Yu,
Peter Richtarik, Ananda Theertha Suresh and Dave Bacon, 2016, \textquotedbl Federated
Learning: Strategies for Improving Communication Efficiency\textquotedbl ,
NIPS Workshop on Private Multi-Party Machine Learning.

\bibitem{b2} Qiang Yang, Yang Liu, Tianjian Chen, and Yongxin Tong.
2019. Federated Machine Learning: Concept and Applications. ACM Trans.
Intell. Syst.

\bibitem{b21} Guan Wang, 2019, Interpret Federated Learning with
Shapley Values. FML Workshop.

\bibitem{b3} H. Brendan McMahan, Eider Moore, Daniel Ramage, and
Blaise Agera y Arcas. 2016. Federated learning of deep networks using
model averaging. CoRR abs/1602.05629 (2016). arxiv:1602.05629 http://arxiv.org/abs/1602.05629.

\bibitem{b4} Albrecht, J. P. 2016. \textquotedbl How the gdpr will
change the world.\textquotedbl{} Eur. Data Prot. L. Rev. 2:287.

\bibitem{b5} Goodman, B., and Flaxman, S. 2016. \textquotedbl European
union regulations on algorithmic decision-making and a right to explanation\textquotedbl .
arXiv preprint arXiv:1606.08813.

\bibitem{b6} Sinno Jialin Pan and Qiang Yang. 2010. A survey on transfer
learning. IEEE Trans. Knowl. Data Eng. 22, 10 (Oct. 2010), 1345--1359.
DOI:https://doi.org/10.1109/TKDE.2009.191

\bibitem{b7} Koh, Pang Wei, and Percy Liang. ``Understanding black-box
predictions via influence functions.'' arXiv preprint arXiv:1703.04730
(2017).

\bibitem{b8} Cook, R. Dennis. ``Detection of influential observation
in linear regression.'' Technometrics 19.1 (1977): 15-18.

\bibitem{b9}Kelwin Fernandes, Jaime S. Cardoso, and Jessica Fernandes.
'Transfer Learning with Partial Observability Applied to Cervical
Cancer Screening.' Iberian Conference on Pattern Recognition and Image
Analysis. Springer International Publishing, 2017.

\bibitem{b10} Achen, Christopher H. Interpreting and using regression.
Vol. 29. Sage, 1982.

\bibitem{b11} Shapley, Lloyd S. "A value for n-person games." Contributions
to the Theory of Games 2.28 (1953): 307-317.

\bibitem{b12} Erik Strumbelj, and Igor Kononenko. "Explaining prediction
models and individual predictions with feature contributions." Knowledge
and information systems 41.3 (2014): 647-665.

\bibitem{b13} Lundberg, Scott M., and Su-In Lee. "A unified approach
to interpreting model predictions." Advances in Neural Information
Processing Systems. 2017.

\bibitem{b14} Lundberg, Scott M., Gabriel G. Erion, and Su-In Lee.
"Consistent individualized feature attribution for tree ensembles."
arXiv preprint arXiv:1802.03888 (2018).

\bibitem{b15} Staniak, Mateusz, and Przemyslaw Biecek. "Explanations
of model predictions with live and breakDown packages." arXiv preprint
arXiv:1804.01955 (2018).

\bibitem{b16} Fujimoto, Katsushige, Ivan Kojadinovic, and Jean-Luc
Marichal. "Axiomatic characterizations of probabilistic and cardinal-probabilistic
interaction indices." Games and Economic Behavior 55.1 (2006): 72-99. 
\end{thebibliography}
\end{document}